\documentclass[sigconf,natbib=true]{acmart}
\usepackage{graphicx}
\usepackage{subcaption}
\usepackage{hyperref} 
\usepackage[capitalise]{cleveref}
\crefformat{equation}{(#2#1#3)} 
\usepackage{enumitem}
\usepackage{amsthm}
\usepackage{amsmath}
\usepackage{url} 
\usepackage{booktabs}  
\usepackage{array}
\usepackage{multirow}
\usepackage{multicol}
\usepackage{adjustbox}
\usepackage{tabularx, booktabs, array, ragged2e}
\usepackage{upgreek}
\usepackage{xcolor}
\usepackage[table]{xcolor}

\definecolor{lightblue}{RGB}{217, 237, 247}
\definecolor{lightgreen}{RGB}{229, 245, 224} 

\newtheorem{theorem}{Theorem} 
\newtheorem*{theorem*}{}
\newtheorem{proposition}{Proposition}  
\newtheorem*{proposition*}{}   
 
\newtheorem*{lemma*}{}
\AtBeginDocument{%
  }

\setcopyright{acmlicensed}
\copyrightyear{2026}
\acmYear{2026}
\acmDOI{XXXXXXX.XXXXXXX}
\acmConference[Conference acronym 'XX]{Make sure to enter the correct
  conference title from your rights confirmation email}{June 03--05,
  2018}{Woodstock, NY}
\acmISBN{978-1-4503-XXXX-X/2018/06}

\begin{document}

\title{Cross-Head Attention Uplift Network with Inverse Propensity Score under Unobserved Confounding}

\author{Haoran Zhang}
\affiliation{%
  \institution{Center for Applied Statistics and School of Statistics, Renmin University of China}
  \city{Beijing}
  \country{China}}
\email{zhrrhz2002@163.com}

\author{Chuanpu Li}
\affiliation{%
  \institution{Alibaba Group}
  \city{Beijing}
  \country{China}}

\author{Yuxin Fu}
\affiliation{%
  \institution{Alibaba Group}
  \city{Beijing}
  \country{China}}

\author{Bin Tong}
\affiliation{%
  \institution{Alibaba Group}
  \city{Beijing}
  \country{China}}

\author{Guan Wang}
\affiliation{%
  \institution{Alibaba Group}
  \city{Beijing}
  \country{China}}

\author{Bo Zheng}
\affiliation{%
  \institution{Alibaba Group}
  \city{Beijing}
  \country{China}}

\author{Feng Zhou$^{*}$}
\thanks{$^{*}$Corresponding author.}
\affiliation{%
  \institution{Center for Applied Statistics and School of Statistics, Renmin University of China}
  \city{Beijing}
  \country{China}}
\email{feng.zhou@ruc.edu.cn}

\renewcommand{\shortauthors}{Trovato et al.}

\begin{abstract}
Uplift modeling, crucial for estimating individual treatment effects (ITE), faces dual challenges: flexibly leveraging inter-group similarity to enhance discriminative power and debiasing under unobserved confounding scenarios. In this paper, we propose the Cross-Head Attention Uplift Network (CHAUN) and Robust Adversarial Inverse Propensity Score (RA-IPS) method to address these limitations. CHAUN employs shared feature embeddings and cross-head attention mechanisms to dynamically integrate treatment-specific and control-specific representations, enhancing inter-group correlation modeling. Theoretically, we prove that access to the true propensity scores ensures ITE identifiability even with unobserved confounders. For practical scenarios lacking true propensity scores, RA-IPS adversarially optimizes propensity weights within constrained uncertainty sets to mitigate bias from unobserved variables. Experiments on public datasets (CRITEO-UPLIFT, LAZADA) and a production e-commerce dataset demonstrate CHAUN’s superiority over state-of-the-art uplift models, achieving a relative improvements of up to $25.6\%$ in QINI scores. RA-IPS further enhances robustness, outperforming standard IPS by $5.4\%$ under unobserved confounding. 
The results validate the effectiveness of our proposed methods in real-world causal inference tasks.
\end{abstract}



\received{20 February 2007}
\received[revised]{12 March 2009}
\received[accepted]{5 June 2009}

\maketitle

\section{Introduction}

Uplift modeling, a causal inference technique for quantifying the marginal effect of treatments on individual behavior, plays a pivotal role in real-world applications such as advertisement~\cite{EUEN,ad}, user growth~\cite{UR,dr-cfr-mim}, and online marketing~\cite{descn,efin}.
The objective of uplift models is to estimate the causal effect of administering a treatment to an individual unit, formally defined as the Individual Treatment Effect (ITE), or uplift.
Uplift modeling enables prediction of ITE at the unit level, thereby facilitating the implementation of precision intervention strategies through data-driven decision frameworks.

We focus on the canonical scenario of binary treatment allocation, whereby observational units are exclusively assigned to either the treatment group or control group. Uplift modeling differs from conventional supervised learning approaches due to the inherent unobservability of ITE labels. For each observational unit, researchers can only empirically measure the factual outcome under either the treatment condition or the control condition, while the corresponding counterfactual outcome remains fundamentally inaccessible. Consequently, the accurate prediction of ITE requires valid counterfactual estimation frameworks for observational units.

One of the essential characteristics of uplift modeling lies in the inherent inter-group similarity between the prediction tasks for treatment group outcomes and control group outcomes. Although the supervision of the treatment group outcomes and control group outcomes is conducted independently, the two prediction tasks remain inherently correlated. Existing methods that incorporate treatment as an input feature~\cite{efin,ECUP} effectively leverage similarity patterns but fundamentally lack explicit mechanisms to sufficiently capture treatment-induced variation. In contrast, multi-head architectures~\cite{cfrnet,dragonnet} based on multi-task learning frameworks lack flexible utilization of such inter-group similarity.
Currently, no flexible architecture simultaneously emphasizes treatment-induced discrimination and exploits inter-group similarity.

Another issue to be addressed in uplift modeling is selection bias~\cite{Imbens_Rubin_2015}, which indicates the presence of confounding variables that simultaneously influence both treatment assignment and outcome responses across observational units. 
Without proper adjustment for confounders, it becomes challenging to isolate the treatment effect from the compounded variance attributable to multifactorial influences in observational causal inference. Traditional statistical research conventionally employs propensity score-based methodologies~\cite{rosenbaum1983,Robins2000,drl} for sample debiasing.
Recent uplift models based on neural network address selection bias through specialized architectures and tailored loss functions~\cite{cfrnet,dr-cfr,descn}. However, these methods fundamentally rely on the strong ignorability assumption~\cite{causality} (a.k.a. unconfoundedness), which posits that all confounding variables are captured within the observed covariates. The strong ignorability assumption serves as a sufficient condition for the identification of ITE~\cite{Imbens_Rubin_2015,pearl}, and typically holds when selection bias originates from observed covariates. 
In complex systems such as bidding advertising systems, the treatment assignment mechanism operates through a multi-stage causal pathway: advertisers' targeting policies determine eligibility for ad exposure, while the impression delivery is mediated by the bidding platform's allocation algorithms. This hierarchical selection process introduces latent confounders that remain unobservable to advertisers.
In scenarios with unobserved confounders, identifying the ITE and debiasing samples remains challenging.


In this paper, we address the two aforementioned challenges. To flexibly exploit inter-group similarity, we propose the {\bfseries C}ross-{\bfseries H}ead {\bfseries A}ttention {\bfseries U}plift {\bfseries N}etwork (CHAUN). The CHAUN framework first learns shared feature embeddings across users through representation alignment, then generates treatment-specific and control-specific deep latent representations via parallel encoding pathways. These dual-branch representations are adaptively fused through attention-weighted integration, where attention mechanisms dynamically determine cross-representation interaction weights. The synthesized representation is ultimately utilized for uplift prediction, enabling counterfactual inference by contrasting treatment-effect responses under treatment and control conditions. To mitigate selection bias, we incorporate inverse propensity score (IPS) weighting into the loss function and impose a global regularization constraint on the propensity scores to stabilize the weights.

To address debiasing in the presence of unobserved confounders, we establish that access to the true propensity scores enables the identification of ITE and unbiased IPS estimator. In practical operational settings where only nominal propensity scores estimated from observed features are accessible, we develop {\bfseries R}obust {\bfseries A}dversarial 
{\bfseries I}nverse {\bfseries P}ropensity {\bfseries S}cores (RA-IPS). We assume that the true propensity score is constrained to vary within a neighborhood of the nominal propensity score, and we optimize for the worst-case scenario to enhance the model’s robustness against unobserved confounding.

Specifically, our contributions are summarized as follows:
\begin{itemize}
    \item We propose CHAUN, a generalized uplift modeling framework that adaptively mediates cross-head interaction through attention gating within a dual-head architecture.
    \item We demonstrate that under unobserved confounding, access to the true propensity scores is sufficient to ensure both the identifiability of ITE and the unbiasedness of IPS estimator.
    \item We propose RA-IPS, a robust inverse propensity score method that addresses unobserved confounders without requiring true propensity scores, leveraging adversarial weighting to enhance stability under unobserved confounding.
    \item We conduct extensive experiments on two public datasets and a production dataset to validate the effectiveness of the proposed CHAUN and RA-IPS methods.
\end{itemize}

\section{Related Work}
In this section, we review existing approaches that leverage inter-group similarity inherent in uplift modeling frameworks, address selection bias, mitigate the impact of unobserved confounders.

\subsection{Exploiting Inter-group Similarity}
Empirical studies~\cite{Xlearner,FlexTenet} suggest that potential outcomes under treatment and control conditions share similar functional structures or model parameters, the ITE as their difference inherently exhibits relatively lower complexity.
The S-Learner~\cite{Xlearner} and frameworks like EFIN~\cite{efin}, ECUP~\cite{ECUP} that encode treatment indicators as model covariates, naturally exploit this prior. 
EUEN~\cite{EUEN} explicitly models the control-group outcome and ITE, then constructs the treatment-group outcome through additive integration of them.
FlexTENet~\cite{FlexTenet} implements partial neuron sharing between layers of the treatment head and control head within its dual-head multi-task architecture. Unlike these methods, CHAUN employs cross-head attention over treatment/control representations to exploit inter-group similarity.

\subsection{Addressing Selection Bias}
In statistical research, practitioners typically employ propensity scores to perform matching~\cite{rosenbaum1983} or inverse probability weighting~\cite{whatif}, thereby constructing pseudo-RCT samples that approximate randomized experimental conditions. In neural network-based approaches, a multitude of methods leverage architectural designs and regularization constraints to mitigate selection bias.
BNN~\cite{BNN} learns a shared feature representation network for both treatment and control groups, while employing Integral Probability Metrics loss to constrain the distributional discrepancy between group representations. 
CFRNet~\cite{cfrnet} retains the shared feature representation, while employing a multi-head architecture with treatment-specific heads to learn group-specific outcome predictions.
Building upon CFRNet, DR-CFR~\cite{dr-cfr} introduces confounder disentanglement and designs propensity score-based weights that account for both factual and counterfactual occurrence probabilities.
DESCN~\cite{descn} and GNUM~\cite{GNUM} enable unified parameter learning across all samples by constructing shared labels for both treatment and control groups. We adopt a weighting approach but introduce a global regularization constraint on the propensity scores to stabilize the weights.

\subsection{Handling Unobserved Confounders}
When a small amount of RCT data is accessible, it can be utilized to correct the bias in large scale observational data~\cite{usurvey,urct1,urct2}.
When RCT data are not possible, causal identification necessitates incorporating structural assumptions or auxiliary information. Common strategies include employing instrumental variables to constrain treatment assignment mechanisms~\cite{deepIV,IV1,IV2}, or identifying proxy variables that correlate with unobserved confounders to enable partial recovery of their latent distributional properties~\cite{proxy,Miao_proxy}.
In the absence of supplementary information, methodologies leveraging sensitivity analyses~\cite{AT-IPS,RD-IPS} enhance robustness against unobserved confounders through adversarial learning.
Our method is based on propensity score sensitivity analysis, but we rigorously derive a more reasonable range for propensity score perturbations.


\section{Preliminaries}
In this section, we extend the Neyman-Rubin potential outcomes framework~\cite{Rubin2005} to incorporate unobserved confounders, to formalize our uplift modeling problem. We then propose solutions under the requirement of leveraging true propensity scores without direct observation of unobserved confounders.

\subsection{Problem Definition}
Assuming we have a dataset $\mathcal{D}$ consisting of $N$ samples $(x_i,u_i,t_i,y_i)$ where $x_i$ denotes observed features, $u_i$ denotes unobserved confounders, $t_i\in \{0,1\}$ denotes the binary treatment, and $y_i$ denotes the outcome response. Each instance of $\mathcal{D}$ is independently sampled from the joint distribution $p(x,u,t,y(0),y(1))$.
Let $\pi(x,u)=P(t=1|x,u)$ denote the true propensity score, which governs the treatment assignment. The unobservability of $u$ renders the true propensity score inaccessible. In practice, we can obtain or estimate the nominal propensity scores $\tilde{\pi}(x)=P(t=1|x)$ based on observed features. The ITE to be estimated is expressed as:
\begingroup
\setlength{\abovedisplayskip}{2pt}
\setlength{\belowdisplayskip}{2pt}
\begin{equation}
    ITE(x) = \mathbb{E}(y(1)-y(0)|x).
    \label{ITEdefinition}
\end{equation}
\endgroup
Recent works~\cite{descn,efin,UMLC} in uplift modeling under the Neyman-Rubin potential outcomes framework~\cite{Rubin2005} universally incorporate three foundational assumptions: 
\begin{enumerate}[label=\arabic*)]
    \item Ignorability: treatment assignment is conditionally independent of potential outcomes given observed covariates, i.e., $y(1),y(0) \perp\!\!\!\perp t|x$. 
    \item Consistency: $y_i$ corresponds to the potential outcome under the administered treatment, i.e., $y_i=t_iy_i(1)+(1-t_i)y_i(0)$.
    \item Overlap: each unit has a non-zero probability to be assigned to each treatment, i.e., $0 < \tilde{\pi}(x)<1 ,\forall x$.
\end{enumerate}
The ignorability assumption entails the absence of unobserved confounders. In the subsequent analysis, we maintain the Consistency and Overlap assumptions throughout, while systematically examining scenarios with and without the presence of unobserved confounders. Notably, when unobserved confounders exist, the overlap assumption generalizes to require $0 < \pi(x,u)<1 ,\forall x,u$.

\subsection{Identifying the ITE}

The ignorability assumption ensures the identifiability of ITE by 
\begin{equation}
    \begin{split}
    ITE(x)& = \mathbb{E}(y(1)-y(0)|x)=\mathbb{E}(y(1)|x)-\mathbb{E}(y(0)|x)
    \\ & = \mathbb{E}(y(1)|x,t=1)-\mathbb{E}(y(0)|x,t=0)
    \\ & = \mathbb{E}(y|x,t=1)-\mathbb{E}(y|x,t=0),
\end{split}
\label{est ite}
\end{equation}
where $\mathbb{E}(y|x,t=1)$ and $\mathbb{E}(y|x,t=0)$ can be estimated from observed data. In \cref{est ite}, the penultimate equality is derived from the Ignorability assumption, and the final equality follows from the Consistency assumption. Note that when the ignorability assumption fails, $\mathbb{E}(y(k)|x)\ne \mathbb{E}(y(k)|x,t=1)$ since $p(y(k)|x)$ and $p(y(k)|x,t=k)$ are not equivalent in distribution, $k=0,1$. Therefore, under such circumstances, we require additional information about the distribution of $u$ to resolve the identifiability of ITE. Diverging from existing approaches that introduce instrumental variables~\cite{deepIV,IV1,IV2} or proxies~\cite{proxy,Miao_proxy}, we draw inspiration from the sufficiency of propensity scores in \cite{dragonnet}, demonstrating that precise knowledge of the true propensity scores $\pi(x,u)$ suffices to guarantee ITE identifiability.
\begin{theorem}
Under the assumption of the presence of unobserved confounders u with known true propensity scores $\pi(x,u)$,
the ITE can be identified through two approaches:
\begin{enumerate}[label=\arabic*),leftmargin=*,  
    labelwidth=1.5em,  
    align=left,         
    itemindent=0pt,     
    labelsep=0.5em]
\item 
\begin{equation}
    ITE(x)=\mathbb{E}\left( \frac{ty}{\pi(x,u)}- \frac{(1-t)y}{1-\pi(x,u)}\bigg| x \right)
\end{equation}

\item If the nominal propensity score $\tilde{\pi}(x)$ is known:
\begin{equation}
    ITE(x)=\mathbb{E}\left(\frac{\tilde{\pi}(x)}{\pi(x,u)}y\bigg|x,t=1\right)-\mathbb{E}\left(\frac{1-\tilde{\pi}(x)}{1-\pi(x,u)}y\bigg|x,t=0\right).
\end{equation}
\end{enumerate}
\label{ITEwithpi}
\end{theorem}
\cref{ITEwithpi} provides a statistical perspective demonstrating that for identifying the ITE in the presence of unobserved confounders $u$, it suffices to understand the joint influence mechanism of $u$ with observed covariates $x$ on treatment assignment, rather than necessitating complete distributional characterization of $u$. Notably, a similar identification principle can be extended to the estimation of the Average Treatment Effect (ATE). Crucially, when true propensity scores $\pi(x,u)$ are accessible, inverse weighting through them enables the construction of a pseudo-population that effectively blocks backdoor paths from unobserved confounders to the treatment. We next explain how to integrate this concept with weighting methods commonly used in machine learning.

\subsection{Inverse Propensity Score}
\label{uips}
 We now analyze selection bias in uplift modeling through a machine learning lens, focusing on how covariate shift between treatment groups impacts model generalization. Let $\tau(x),\mu_0(x),\mu_1(x)$ denote the $ITE(x),\mathbb{E}(y|x,t=0)$ and $\mathbb{E}(y|x,t=1)$ fitted by the model respectively, $T=\{i:t_i=1\},C=\{i:t_i=0\}$ are treatment group and control group. Selection bias in uplift modeling refers to the phenomenon where the model-learned $\mu_0(x),\mu_1(x)$ are trained exclusively on the treatment and control group samples respectively, leading to poor generalization to the overall population distribution. Analogous issues are prevalent in recommendation systems~\cite{rec_as_trt,unbiasrec} and CTR prediction~\cite{ESMM,DCMT}.
Formally, the bias arises from the non-equivalence between the following two loss functions:
\begin{align}
    &\mathcal{L}_{ideal} = \frac1N\sum_{i=1}^N (l(\mu_0(x_i),y_i(0))+   l(\mu_1(x_i),y_i(1))),\\
    &\mathcal{L}_{fact} = \frac2N(\sum_{i\in C} l(\mu_0(x_i),y_i(0))+  \sum_{i\in T} l(\mu_1(x_i),y_i(1))),
\end{align}
where $l(\cdot,\cdot)$ denotes the per-sample loss function, typically specified as mean squared error or cross-entropy. While our objective is to optimize the ideal loss $\mathcal{L}_{ideal}$, the fundamental unobservability of counterfactual outcomes necessitates optimizing the factual loss $\mathcal{L}_{fact}$ in practice, where $\mathbb{E}_t(\mathcal{L}_{fact})\ne \mathcal{L}_{ideal} $ holds in general due to selection bias~\cite{rec_as_trt}. To tackle this issue, the established methodology applies IPS~\cite{Imbens_Rubin_2015,samplingT} as a reweighting mechanism for the loss function, where propensity weights compensate for selection bias through importance sampling:
\[
\mathcal{L}_{IPS}=\frac1N\sum_{i=1}^N \bigg(\frac{(1-t_i)l(\mu_0(x_i),y_i(0))}{1-\tilde{\pi}(x_i)}+  \frac{t_il(\mu_1(x_i),y_i(1))}{\tilde{\pi}(x_i)}\bigg).
\]
 The statistical validity of the IPS method originates from its construction of an unbiased estimator for $\mathcal{L}_{ideal}$ under the ignorability assumption, i.e.,
\begin{equation}
\label{unbias}
    \mathbb{E}_t(\mathcal{L}_{IPS}) = \mathcal{L}_{ideal}.
\end{equation}
 However, in the presence of unobserved confounders, \cref{unbias} below no longer holds universally due to the discrepancy between nominal and true propensity scores, i.e., $\pi(x_i,u_i)\ne\tilde{\pi}(x_i)$. Unbiasedness can be preserved if and only if the weighting scheme employs the true propensity score~\cite{RD-IPS}. 
\begin{theorem}
\label{unbias}
Under the assumption of the presence of unobserved confounders u, let $\mathcal{L}_{true-IPS}$ denote the IPS estimator using the true propensity score $\pi(x_i,u_i)$, formally expressed as: 
\[
\mathcal{L}_{true-IPS}=\frac1N\sum_{i=1}^N \bigg(\frac{(1-t_i)l(\mu_0(x_i),y_i(0))}{1-\pi(x_i,u_i)}+  \frac{t_il(\mu_1(x_i),y_i(1))}{\pi(x_i,u_i)}\bigg).
\]
Then it is an unbiased estimator of $\mathcal{L}_{ideal}$. 
\end{theorem}
 
\section{Methodology}
In this section, we elaborate on the proposed method consisting of two distinct methodological components: (1) CHAUN, a general-purpose uplift network architecture for binary treatment settings, and (2) RA-IPS, a novel approach that enhances the robustness of IPS estimation under unobserved confounding. CHAUN predicts nominal propensity scores and uses them for weighting; this approach is fully valid in the absence of unobserved confounders, but must be combined with RA-IPS when such confounders are present.

\subsection{CHAUN}
The architecture of the proposed CHAUN is illustrated in \cref{archi}. The architecture primarily comprises three components: a Shared Feature Embedding Layer, a Propensity Learner, and an Outcome Learner. In the subsequent subsections, we systematically elaborate on the design principles and implementation details of each module.

\begin{figure*}[ht]
  \centering
  \includegraphics[
    width=0.95\textwidth,      
    trim={5pt 172pt 3pt 112pt}, 
    clip                      
  ]{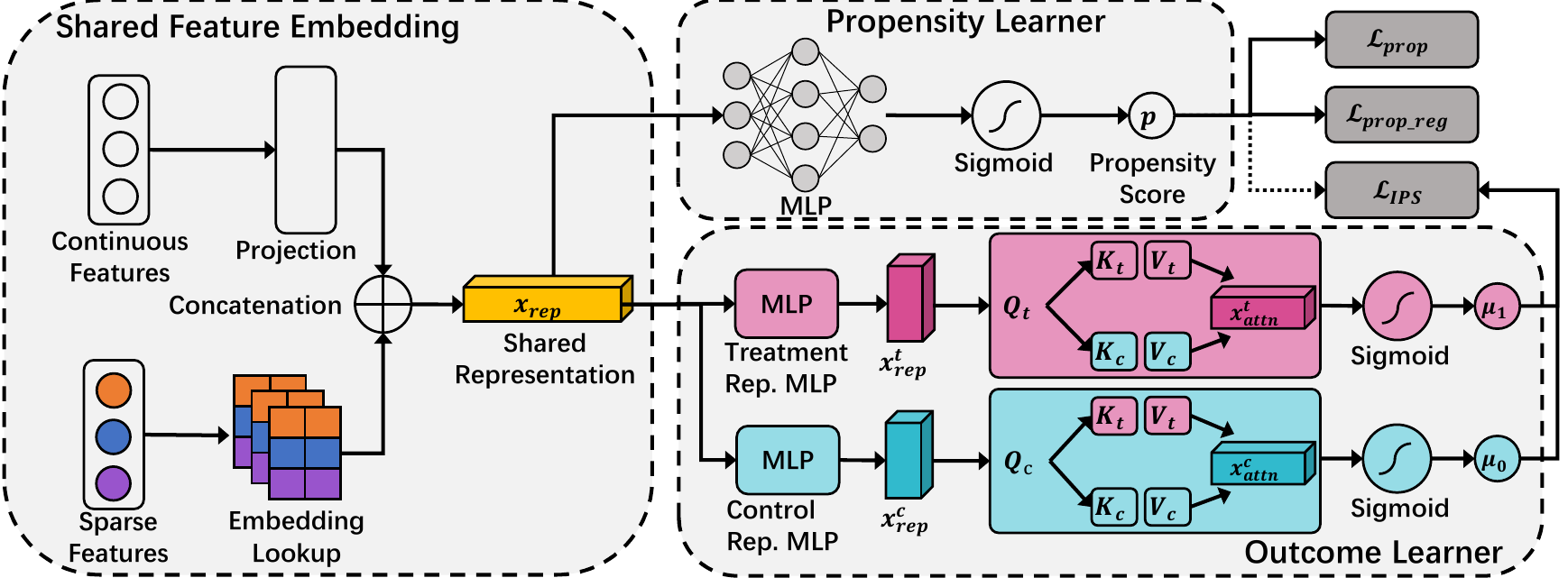}   
  \vspace{-5pt}
  \caption{The architecture overview of CHAUN. }
  \label{archi}
  \vspace{-5pt}
\end{figure*}

\subsubsection{Shared Feature Embedding}
For an instance with observed non-treatment features $x_i$, assumed to consist of both continuous and sparse features, we process them separately. For continuous features, we apply a dimension-preserving projection with learnable affine transformations, while for sparse features, we initialize dedicated embedding tables and retrieve the corresponding embeddings via lookup operations based on their categorical values. The processed feature embeddings are concatenated into a unified vector $x_{rep}$, serving as the shared input to both the Propensity Learner and Outcome Learner.


\subsubsection{Propensity Learner}
Upon obtaining the instance representation $x_{rep}$, the Propensity Learner processes it through a multi-layer perceptron (MLP) architecture comprising stacked fully-connected layers interleaved with activation functions. The network culminates in a linear projection layer that reduces dimensionality to 1, followed by a sigmoid activation function to yield the final propensity score prediction. Let the MLP contain $n$ hidden layers, the Propensity Learner is formulated as:
\begin{align*}
x_{hidden}^{(k)}&=\text{ReLU}(\text{Linear}(x_{hidden}^{(k-1)}))\in\mathbb{R}^d, k=1,\cdots,n,\\
\hat{p}&=\text{Sigmoid}(\text{Linear}(x_{hidden}^{(n)})) \in (0,1),
\end{align*}
where $x_{hidden}^{(k)}$ denotes the output of the $k$-th layer, with the initial hidden state defined as $x_{hidden}^{(0)}=x_{rep}$.

\subsubsection{Outcome Learner} 
In contrast to conventional dual-head architectures that either independently predict potential outcomes of treatment and control groups from shared feature representations or impose regularization constraints to differentiate the heads, we propose an adaptive interaction mechanism that establishes dynamic coupling between the dual learners through attention-based gating. The Outcome Learner initially processes $x_{rep}$ through two MLPs to derive potential outcome representations under both treated and control conditions, formally expressed as:
\begin{align*}
    x_{rep}^c=\text{MLP}(x_{rep})\in \mathbb{R}^d,\\
    x_{rep}^t=\text{MLP}(x_{rep})\in \mathbb{R}^d,
\end{align*}
where the $\text{MLP}$ is composed of multiple fully-connected layers interleaved with nonlinear activation functions. Rather than directly utilizing the two representations to generate corresponding outputs, we implement an attention-gated fusion mechanism to produce the final prediction through adaptive weighted combination. The fused representation is then processed by a lightweight output layer to generate the final prediction. For binary response variables, this process can be formally expressed as: 
\begin{align*}
    x_{attn}^c = \frac{\exp(q^ck^c)v^c+\exp(q^ck^t)v^t}{\exp(q^ck^c)+\exp(q^ck^t)}\in \mathbb{R}^d,\\
    x_{attn}^t = \frac{\exp(q^tk^c)v^c+\exp(q^tk^t)v^t}{\exp(q^tk^c)+\exp(q^tk^t)}\in \mathbb{R}^d,\\
    \mu_0 = \text{Sigmoid}(\text{Linear}(x_{attn}^c)) \in (0,1),\\
    \mu_1 = \text{Sigmoid}(\text{Linear}(x_{attn}^t)) \in (0,1),
\end{align*}
where $q^c,k^c,v^c$ and $q^t,k^t,v^t$ are derived from $x_{rep}^c$ and $x_{rep}^t$ via learnable linear projections respectively.

\subsubsection{Loss} 
To address sample selection bias, we incorporate IPS into the training loss computation, which requires accurate estimation of propensity scores. However, naively learning these scores through binary classification of treatment assignment labels may yield pathological solutions in which $\hat{p}(x_i)=t_i$, a seemingly perfect but overlap-violating estimator. For stabilizing the inverse probability weights, a global regularization constraint is imposed on the propensity score estimation, grounded in the theoretical framework established by the following proposition. 
\begin{proposition}
\label{IPSres}
    Assuming the covariates x and treatment assignment t are governed by the joint distribution p(x, t) (whether the Ignorability assumption holds or not) with the overlap condition satisfied, the following holds: 
    \begin{align}
       \mathbb{E}\left(\frac{t}{\tilde{\pi}(x)}\right)= \mathbb{E}\left(\frac{1-t}{1-\tilde{\pi}(x)}\right)=1. 
    \end{align}
\end{proposition}
Building upon \cref{IPSres}, the supervised loss function for propensity score estimation comprises two key components:
\begin{gather*}
    \mathcal{L}_{prop}=\frac{1}{N}\sum_{i=1}^{N}l(\hat{p}(x_i),t_i),\\
    \mathcal{L}_{prop\_reg}=\left(\frac{1}{N}\sum_{i=1}^{N}\frac{t_i}{\hat{p}(x_i)}-1\right)^2+\left(\frac{1}{N}\sum_{i=1}^{N}\frac{1-t_i}{1-\hat{p}(x_i)}-1\right)^2, 
\end{gather*}
where $l(\cdot,\cdot)$ denotes the cross-entropy loss function.
Computed propensity scores are gradient-detached and used to reweight the potential outcome prediction loss:
\begin{equation*}
    \mathcal{L}_{IPS}=\frac{1}{N}\sum_{i=1}^{N}\left(\frac{t_il(\mu_1(x_i),y_i)}{\hat{p}(x_i)}+\frac{(1-t_i)l(\mu_0(x_i),y_i)}{1-\hat{p}(x_i)}\right).
\end{equation*}
The final optimization objective function is expressed as:
\begin{equation*}
\mathcal{L}=\mathcal{L}_{IPS}+\alpha \mathcal{L}_{prop}+\beta \mathcal{L}_{prop\_reg},
\end{equation*}
where $\alpha$ and $\beta$ are hyperparameters.

\subsection{RA-IPS}
\label{RA-IPS}
As demonstrated in \cref{uips}, the IPS estimator based on nominal propensity scores $\tilde{\pi}(x)$ fails to achieve effective bias correction under the presence of unobserved confounders. In contrast, with access to true propensity scores $\pi(x,u)$, debiasing can be achieved even without accurate information about the unobserved confounders. However, true propensity scores are generally unobservable in practice. To address this fundamental limitation, we develop a robust IPS estimation framework that incorporates sensitivity analysis principles, specifically designed to improve robustness to unobserved confounding via adversarial propensity.

The nominal propensity score 
$\tilde{\pi}(x)$ can be formally characterized as the conditional expectation of the true propensity score $\pi(x,u)$ given the observed covariates $x$,
\[
\label{true prop}
\tilde{\pi}(x)= \int \pi(x,u)p(u|x)du.
\]
While nominal propensity scores cannot identify the true latent propensity scores, they constrain the feasible domain of the true propensity score.
\begin{proposition}
\label{true prop limit}
Suppose that $a(x)<\pi(x,u)<1-a(x)$, then the following inequalities each hold with probability at most $\eta$, 
\begin{align*}
    |\frac{1}{\pi(x,u)}-\frac{1}{\tilde{\pi}(x)}|&\ge \frac{1-2a(x)}{2a(x)\tilde{\pi}(x)\sqrt\eta}, \\
|\frac{1}{1-\pi(x,u)}-\frac{1}{1-\tilde{\pi}(x)}|&\ge \frac{1-2a(x)}{2(1-a(x))\tilde{\pi}(x)\sqrt\eta}. 
\end{align*}
\end{proposition}
Under the assurance of \cref{true prop limit}, it is statistically valid to posit that the true propensity score resides within a neighborhood of the nominal propensity score with high probability. Therefore, a reasonable assumption is that the true propensity score is obtained by perturbing the nominal propensity score. We adhere to the theoretical framework of RD-IPS~\cite{RD-IPS} by introducing a hyperparameter to regulate the impact of unobserved confounders on the logit-transformed propensity scores. Formally, the true propensity score can be mathematically expressed as follows:
\begin{align*}
    \tilde{\pi}(x)&=\text{Sigmoid}(f(x)),\\
    \pi(x,u)&=\text{Sigmoid}(f(x)+g(u)).
\end{align*}
 Assuming $g(u)$ is bounded by $|g(u)|<\log(\Gamma)$, where $\Gamma\ge1$ is a hyperparameter corresponding to the strength of unobserved confounders. This leads to the constraints on the true weights:
 \[
a(x,t) \le w(x,u,t)\le b(x,t),
 \]
where $a(x,t)=1+(\tilde{w}(x,t)-1)\frac{1}{\Gamma},b(x,t)=1+(\tilde{w}(x,t)-1)\Gamma$, and $\tilde{w}(x,t)=\frac{t}{\tilde{\pi}(x)}+\frac{1-t}{1-\tilde{\pi}(x)},w(x,u,t)=\frac{t}{\pi(x,u)}+\frac{1-t}{1-\pi(x,u)} $denote the nominal weight and true weight respectively. Therefore, the RD-IPS method formulates the following uncertainty set: 
\[
\mathcal{W} = \{W\in\mathbb{R}_+^{N}: a(x_i,t_i)\le w_i\le b(x_i,t_i) \},
\]
and maximize the $\mathcal{L}_{IPS}$ objective over all possible configurations of true weights to optimize for the worst-case scenario, thereby enhancing the robustness. The optimization objective is: 
\begin{equation}
\label{rdips}
    \mathcal{L}_{RD-IPS}=\max_{W\in \mathcal{W}}\frac{1}{N}\sum_{i=1}^Nw_il(\mu_{t_i}(x_i),y_i(t_i)). 
\end{equation}

However, since the loss terms for all samples are non-negative, the maximization operation trivially leads to each weight attaining its upper bound within the uncertainty set. Such weight configurations remain structurally infeasible under real-world data distributions, a limitation arising from the absence of global constraints on weight combinations in $\mathcal{W}$.

\begin{proposition}
\label{UIPSres}
    Assuming the covariates $x$, unobserved confounders $u$ and treatment assignment $t$ are governed by the joint distribution $p(x,u,t)$ with the overlap condition satisfied, the following holds: 
    \begin{align}
       \mathbb{E}\left(\frac{t}{\pi(x,u)}\right)= \mathbb{E}\left(\frac{1-t}{1-\pi(x,u)}\right)=1. 
    \end{align}
\end{proposition}

In large-scale observational datasets, a properly specified nominal propensity score and true propensity score should conform to the theoretical guarantees established in \cref{IPSres} and \cref{UIPSres}. According to the weak law of large numbers, we have
\begin{align*}
    \lim_{N \to \infty}\frac{1}{N}\sum_{i=1}^N w(x_i,u_i,t_i)&\xrightarrow{p}2, \\
    \lim_{N \to \infty}\frac{1}{N}\sum_{i=1}^N \tilde{w}(x_i,t_i)&\xrightarrow{p}2. 
\end{align*}
Thus,
\begin{align}
\label{restrict}
    &\lim_{N \to \infty}\frac{1}{N}\sum_{i=1}^N (w(x_i,u_i,t_i)-\tilde{w}(x_i,t_i))\xrightarrow{p}0. 
\end{align}
However, the solution for $W$ derived from \cref{rdips} would result in
\begin{align*}
    \frac{1}{N}\sum_{i=1}^N (w(x_i,u_i,t_i)-\tilde{w}(x_i,t_i))=(\Gamma-1)\left(\frac{1}{N}\sum_{i=1}^N \tilde{w}(x_i,t_i)-1\right), 
\end{align*}
which fundamentally violates \cref{restrict} if $\Gamma>1$ and $N$ is large enough. 
Building upon the constraints on true weights and nominal weights established in \cref{restrict}, we propose the following rigorously justified uncertainty set:
\begin{equation}
\label{set}
    \mathcal{W}\cap  \{W\in\mathbb{R}_+^{N}: \frac{1}{N}\sum_{i=1}^N (w_i-\tilde{w}_i)\le\epsilon_N\},
\end{equation}
where $\epsilon_N>0,\epsilon_N \to 0$ as $N \to \infty$. Therefore, in our adversarial learning framework, we incorporate such constraints as a regularization term to ensure feasible weight configurations during robust optimization. The final objective function to be optimized is formulated as follows:
\begin{equation*}
    \hat{W}=\operatorname*{argmax}\limits_{W\in \mathcal{W}}\frac{1}{N}\sum_{i=1}^Nw_il(\mu_{t_i}(x_i),y_i(t_i))+\lambda(\frac{1}{N}\sum_{i=1}^N (w_i-\tilde{w}_i))^2, 
\end{equation*}
\begin{equation}
    \mathcal{L}_{RA-IPS}=\frac{1}{N}\sum_{i=1}^N\hat{w}_il(\mu_{t_i}(x_i),y_i(t_i)). 
\end{equation}

By empirically constraining the relationship between true weights and nominal weights in large-scale observational data, we establish a necessary condition for $W$ being the true weight configurations. Within a further constrained uncertainty set, we optimize for the worst-case scenario among all admissible $W$, thereby ensuring that the model's learned potential outcomes remain robust when the impact of unobserved confounders on treatment assignment is bounded within a specified range.

To analyze the generalization error, we define a unified hypothesis class $\mathcal{F}$ consisting of joint prediction functions $f_\psi: \mathcal{X} \times \{0, 1\} \to \mathbb{R}$, where each $f_\psi \in \mathcal{F}$ is given by $f(x, t) = t \cdot \mu_1(x) + (1 - t) \cdot \mu_0(x)$, and the pair $(\mu_0, \mu_1)$ is parameterized by the model architecture. The empirical Rademacher complexity is then defined as  
\[
\mathcal{R}(\mathcal{F}) = \mathbb{E}_{\sigma\sim\{-1,1\}^N} \left[ \sup_{f \in \mathcal{F}} \frac{1}{N} \sum_{i=1}^N \sigma_i \, l(f(x_i, t_i), y_i) \right],
\]  
where $\sigma_i$ are i.i.d. Rademacher random variables ($\mathbb{P}(\sigma_i = \pm 1) = \frac12$). Let $l_i,w_i$ denote the observable loss and true weight for the $i$-th sample, i.e., $l_i=l(\mu_{t_i}(x_i),y_i(t_i)), w_i=\frac{t_i}{\pi(x_i,u_i)}+\frac{1-t_i}{1-\pi(x_i,u_i)}$.

\begin{theorem}
\label{bound}
    (Generalization Bound) Suppose that the true weight configurations lie within the uncertainty set in \cref{set}, and $l_i \le C_1, w_i\le C_2, \hat{w}_i\le C_2, \forall i$. For any $f_\phi \in \mathcal{F}$, we have the following with probability at least $1-\eta$: 
\begin{equation*}
    |\mathcal{L}_{RA-IPS}(\phi) - \mathcal{L}_{ideal}(\phi)| 
    \leq 2 C_2 \mathcal{R}(\mathcal{F}) + C_1 C_2(1+ \sqrt{\frac{\log(2/\eta)}{2N}}). 
\end{equation*}
\end{theorem}

\section{Experiments}
In this section, we conduct extensive experiments to answer the following research
questions:
\begin{enumerate}[label={\LARGE$\bullet$}, leftmargin=*, itemsep=0pt]
    \item \textbf{RQ1:} How does our CHAUN outperform different baseline uplift models?
    \item \textbf{RQ2:} How does ablation or substitution of CHAUN's core architectural components impact model performance?
    \item \textbf{RQ3:} Does RA-IPS outperform the standard IPS estimator under unobserved confounding?
\end{enumerate}

\subsection{Experimental Setup}

\subsubsection{Datasets}
We conduct experiments on two widely used large-scale binary-treatment uplift datasets, {\bfseries CRITEO-UPLIFT}~\cite{criteo} and {\bfseries LAZADA}~\cite{descn} as well as a proprietary off-site advertising campaign dataset ({\bfseries Production}). For CRITEO-UPLIFT, we randomly split it into training and evaluation sets with an 8/2 ratio and select visit as the target label. Since treatment assignment is randomized, exposure within the treatment group exhibits selection bias. Therefore, we use exposure rather than treatment assignment as the treatment label. Subsequently, we implement propensity score matching (PSM) to construct a pseudo-RCT evaluation set.
In the Production dataset, treatment assignment is influenced by decisions made by external media platforms, making it infeasible to collect RCT data as an evaluation set. Therefore, we construct the evaluation set using propensity score matching (PSM). These external media decisions are driven by a set of features that cannot be fully obtained due to cost, risk control, permission constraints, and other practical limitations. We thus regard these features as unobserved confounders. Nevertheless, we were able to obtain these features for a subset of samples and used them in the PSM procedure to ensure that the resulting evaluation set is a fully pseudo-RCT. Key statistics of these three datasets used in our experiments are presented in \cref{datasets}.
\begin{table*}[ht]  
  \centering
  \caption{The statistics of datasets.}
  \vspace{-5pt}
  \label{datasets}
  \small
  \begin{tabularx}{0.75\textwidth}{
      >{\hsize=2.2\hsize\Centering\arraybackslash}X
      *{3}{|>{\hsize=0.8\hsize\Centering\arraybackslash}X
            >{\hsize=0.8\hsize\Centering\arraybackslash}X}
    }
    \toprule[1.2pt]
    \textbf{Dataset} &
    \multicolumn{2}{c|}{\textbf{CRITEO-UPLIFT}} &
    \multicolumn{2}{c|}{\textbf{LAZADA}} &
    \multicolumn{2}{c}{\textbf{Production}} \\
    \cmidrule(lr){2-3} \cmidrule(lr){4-5} \cmidrule(lr){6-7}
    Split & Train & Test & Train & Test & Train & Test \\
    \midrule
    Size              & 11183673 & 112788 & 926669 & 181669 & 21350813 & 11963551 \\
    Features          & 12       & 12     & 83     & 83     & 93       & 93 \\
    Treatment Ratio   & $3.07\%$      & $50.00\%$  & $22.17\%$ & $52.12\%$    & $53.78\%$      & $49.99\%$ \\
    Average Conversion Rate & $4.7\%$ & $24.12\%$    & $2.00\%$  & $3.52\%$    & $30.01\%$ & $40.60\%$ \\
    Relative Average Uplift & $1070.04\%$ & $180.02\%$ & $502.13\%$ & $11.11\%$    & $109.27\%$  & $2.86\%$ \\
    Average Uplift & $37.88\%$ & $22.85\%$ & $4.72\%$ & $0.37\%$  & $20.66\%$ & $1.15\%$ \\
    \bottomrule[1.2pt]
  \end{tabularx}
\end{table*}

\subsubsection{Metrics}
We employ three metrics commonly used in evaluating uplift modeling ranking performance, uplift score at first $h$ percentile (LIFT$@h$, we set $h$ to 30), normalized area under the uplift curve (AUUC), normalized area under the QINI curve (QINI) and normalized principled uplift curve (PUC)~\cite{ptonet} . The AUUC and QINI are calculated using functions from python package scikit-uplift.

\subsubsection{Baselines} We compare our CHAUN with many neural network based uplift models including {\bfseries S-Learner~\cite{Xlearner}}, {\bfseries T-Learner~\cite{Xlearner}}, {\bfseries TARNet~\cite{cfrnet}}, {\bfseries CFRNet~\cite{cfrnet}}, {\bfseries DragonNet~\cite{dragonnet}}, {\bfseries CEVAE~\cite{CEVAE}}, {\bfseries FlexTENet~\cite{FlexTenet}}, {\bfseries EUEN~\cite{EUEN}}, {\bfseries DESCN~\cite{descn}}, {\bfseries EFIN~\cite{efin}}. And we compare our proposed RA-IPS with {\bfseries RD-IPS~\cite{RD-IPS}}.

\subsubsection{Implementation Details} All experiments are implemented using PyTorch 2.1.0~\cite{Pytorch} and conducted on a single Tesla V100 GPU with 32GB memory. We use Adam~\cite{adam} as the optimizer and a maximum iteration count of 50. To prevent overfitting, we implement an L2 weight decay of $1e^{-4}$ along with an early stopping strategy with a patience of 10. For each input batch, continuous features are first processed through batch normalization. Using the QINI metric as the criterion, we search for the optimal primary model and training hyperparameters within the parameter ranges described in \cref{hyper}. 
The code of experiments will be provided after the paper is accepted.
\begin{table}[ht]
  \centering
  \caption{Main model and training hyperparameters and their value range.}
  \vspace{-5pt}
  \label{hyper}
  \begin{tabular}{ccc}  
    \toprule[1.2pt]    
    \textbf{Name} & \textbf{Range} & \textbf{Functionality} \\
    \midrule           
    $d$ & $\{2^6,2^7,2^8,2^9\}$ & hidden dimensions \\
    $n$ & $\{1,2,3,4\}$ & number of layers \\
    $bs$ & $\{2^8,2^9,2^{10},2^{11}\}$ & batch size\\
    $lr$ & $\{5e^{-3},1e^{-3},5e^{-4},1e^{-4}\}$ & learning rate \\
    $\alpha,\beta,\lambda$ & $\{10, 1, 1e^{-1},1e^{-2}\}$ & weight of multi-task loss \\
    \bottomrule[1.2pt] 
  \end{tabular}
\end{table}

\subsection{Overall Performance Comparison}

\begin{table*}[t] 
  \centering
  \caption{Overall performance comparison between our CHAUN and other baseline models. "Random" refers to the calculated expected metric value after random scoring and ranking. All experiments are conducted with 5 random seeds, with metric averages computed for robustness evaluation. Higher LIFT$@30$, AUUC, QINI, PUC indicate superior uplift modeling capability. The best results of each benchmark are in \cellcolor{lightblue}{\textbf{bold}} and the second best are \cellcolor{lightgreen}{\underline{underlined}}.}
  \vspace{-9pt}
  \label{performance all}
  \begin{adjustbox}{width=0.96\textwidth, keepaspectratio}
  \small 
\begin{tabularx}{\textwidth}{
    >{\hsize=1.2\hsize\centering\arraybackslash}X
    *{3}{
    | >{\hsize=0.85\hsize\centering\arraybackslash}X  
      >{\hsize=0.85\hsize\centering\arraybackslash}X  
      >{\hsize=0.85\hsize\centering\arraybackslash}X  
      >{\hsize=0.85\hsize\centering\arraybackslash}X }}
\toprule[1.2pt]
\multirow{2}{*}{\textbf{Method}} & 
\multicolumn{4}{c|}{\textbf{CRITEO-UPLIFT}} & 
\multicolumn{4}{c|}{\textbf{LAZADA}} & 
\multicolumn{4}{c}{\textbf{Production}}\\
\cmidrule(lr){2-5} \cmidrule(lr){6-9} \cmidrule(lr){10-13}
& LIFT$@30$ & AUUC & QINI & PUC & LIFT$@30$ & AUUC & QINI & PUC & LIFT$@30$ & AUUC & QINI & PUC \\
\midrule 
Random & 0.2285 & 0.0000 & 0.0000 & 0.0000 & 0.0037 & 0.0000 & 0.0000 & 0.0000 & 0.0114 & 0.0000 & 0.0000 & 0.0000 \\
S-Learner      & 0.3820 & 0.1541 & 0.1848 & 0.1585 & 0.0080 & 0.0033 & 0.0236 & 0.0060 & 0.0188 & 0.0083 & 0.0074 & 0.0073 \\
T-Learner      & 0.3844 & 0.1551 & 0.1859 & 0.1591 & 0.0078 & 0.0024 & 0.0172 & \cellcolor{lightblue}{\textbf{0.0099}} & 0.0181 & 0.0099 & 0.0086 & 0.0131 \\
TARNet         & 0.3865 & 0.1701 & 0.1890 & 0.1625 & \cellcolor{lightgreen}{\underline{0.0084}} & \cellcolor{lightgreen}{\underline{0.0042}} & \cellcolor{lightgreen}{\underline{0.0302}} & 0.0064 & 0.0234 & 0.0127 & 0.0113 & 0.0120 \\
CFRNet         & 0.3683 & 0.1463 & 0.1754 & 0.1505 & 0.0064 & 0.0025 & 0.0176 & 0.0052 & 0.0236 & 0.0131 & 0.0116 & 0.0122 \\
DragonNet      & \cellcolor{lightgreen}{\underline{0.3846}} & \cellcolor{lightgreen}{\underline{0.1595}} & \cellcolor{lightgreen}{\underline{0.1898}} & \cellcolor{lightgreen}{\underline{0.1631}} & 0.0082 & 0.0041 & 0.0250 & 0.0067 & 0.0228 & 0.0118 & 0.0106 & 0.0113 \\
CEVAE          & 0.3605 & 0.1308 & 0.1585 & 0.1392 & 0.0072 & 0.0029 & 0.0200 & 0.0032 & 0.0256 & 0.0129 & 0.0118 & 0.0124 \\
EUEN           & 0.3665 & 0.1349 & 0.1616 & 0.1395 & 0.0081 & 0.0031 & 0.0223 & 0.0030 & 0.0245 & 0.0111 & 0.0097 & 0.0101 \\
FlexTENet      & 0.3648 & 0.1418 & 0.1708 & 0.1465 & 0.0080 & 0.0036 & 0.0253 & 0.0054 & 0.0166 & 0.0133 & 0.0117 & 0.0091 \\
DESCN          & 0.3841 & 0.1516 & 0.1820 & 0.1560 & 0.0076 & 0.0028 & 0.0190 & 0.0031 & \cellcolor{lightblue}{\textbf{0.0286}} & \cellcolor{lightgreen}{\underline{0.0137}} & \cellcolor{lightgreen}{\underline{0.0122}} & \cellcolor{lightblue}{\textbf{0.0150}} \\
EFIN           & 0.3516 & 0.1372 & 0.1651 & 0.1411 & 0.0077 & 0.0032 & 0.0230 & 0.0059 & 0.0221 & 0.0109 & 0.0092 & 0.0099  \\
CHAUN          & \cellcolor{lightblue}{\textbf{0.3904}} & \cellcolor{lightblue}{\textbf{0.1620}} & \cellcolor{lightblue}{\textbf{0.1955}} & \cellcolor{lightblue}{\textbf{0.1664}} & \cellcolor{lightblue}{\textbf{0.0087}} & \cellcolor{lightblue}{\textbf{0.0044}} & \cellcolor{lightblue}{\textbf{0.0314}} & \cellcolor{lightgreen}{\underline{0.0091}} & \cellcolor{lightgreen}{\underline{0.0261}} & \cellcolor{lightblue}{\textbf{0.0144}} & \cellcolor{lightblue}{\textbf{0.0130}} & \cellcolor{lightgreen}{\underline{0.0144}} \\
\bottomrule[1.2pt]
\end{tabularx}
\end{adjustbox}
\end{table*}

We report the main experimental results of CHAUN and other baseline models on three real-world datasets in \cref{performance all}, and visualize the uplift curves in \cref{curve}. According to these results, we have the following observations:
1) AUUC and QINI generally exhibit consistent monotonicity, whereas PUC demonstrates divergent patterns of similarity under specific conditions. This characteristic enables them to evaluate uplift models from complementary dimensions.
2) On two public datasets, S-Learner and T-Learner outperform some methods with complex architectures or extra regularization.  
3) TARNet and DragonNet achieve solid performance across all datasets, demonstrating that shared feature embedding layers, incorporating propensity score prediction as an auxiliary task, and properly calibrated weighting mechanisms can reliably improve model effectiveness.
4) Our proposed CHAUN demonstrates consistently superior performance across all three evaluated datasets. Specifically, it secures top-1 rankings on 9 out of 12 evaluation metrics and top-2 positions on 3 metrics. Notably, on two public benchmark datasets, CHAUN achieves state-of-the-art results across nearly all measurement dimensions. As evidenced by Figure 3, CHAUN's uplift curve demonstrates strong monotonicity, achieving nearly the highest uplift in the top $20\%$ of samples and the lowest uplift in the bottom $20\%$ of samples, which validates its exceptional ranking capability for treatment effect stratification.
 On the Production dataset, only DESCN achieves comparable performance to CHAUN through domain-specific assumptions and customized loss designs, with gains being highly dataset-sensitive. Compared with the baseline DragonNet, CHAUN achieves QINI improvements ranging from $3\%$ to $25.6\%$ across three distinct datasets, demonstrating robust performance gains under varying data conditions.

For the CRITEO-UPLIFT and LAZADA datasets, where treatment assignment is fully determined by observed covariates, we randomly selected subsets of features and employed permutation tests to verify that they simultaneously affected both treatment assignment and potential outcomes, thereby ensuring that they indeed constituted confounders. We then masked these features to simulate unobserved confounding scenarios. Subsequently, we evaluated the performance of IPS, RD-IPS, and RA-IPS using two architectures: 1) a network with shared embedding layers and separate heads for propensity score, treated outcome, and control outcome prediction (architecturally analogous to DragonNet but distinguished by a multi-layer propensity score head, we formally denote this framework as BaseNet); 2) the proposed CHAUN. As analytically demonstrated in \cref{RA-IPS}, the weight configurations constructed in the RD-IPS method do not align with the worst-case scenario within the theoretical bounds of the confounding assumptions controlled by $\Gamma$. Instead, they amplify individual sample weights based on their nominal propensity scores. This approach effectively amplifies the scale of the IPS loss, which may empirically enhance practical optimization dynamics, but fundamentally deviates from the theoretically grounded adversarial robustness framework. Therefore, as demonstrated by the average metrics computed through repeated experiments in \cref{performance RA-ips}, RD-IPS exhibits no significant improvement over IPS across all benchmarks. In contrast, our proposed RA-IPS method rigorously identifies and optimizes against the worst-case scenarios induced by unobserved confounders affecting treatment assignment. Under most evaluation metrics, RA-IPS achieves significant performance improvements during causal effect estimation, attaining up to a $5.4\%$ enhancement in QINI compared to IPS.

\begin{table*}[ht]
  \centering
  \caption{Performance comparison of IPS, RD-IPS, and RA-IPS methods. All experiments were conducted with 5 random seeds, with metric averages computed for robustness evaluation. Best results of each benchmark are in \cellcolor{lightblue}{\textbf{bold}}.}
  \vspace{-9pt}
  \label{performance RA-ips}
  \begin{adjustbox}{width=0.96\textwidth, keepaspectratio}
  \small 
\begin{tabularx}{\textwidth}{
    >{\hsize=2.2\hsize\centering\arraybackslash}X  
    *{3}{                                         
    | >{\hsize=0.9\hsize\centering\arraybackslash}X  
      >{\hsize=0.9\hsize\centering\arraybackslash}X  
      >{\hsize=0.9\hsize\centering\arraybackslash}X  
      >{\hsize=0.9\hsize\centering\arraybackslash}X }}
\toprule[1.2pt]
\multirow{2}{*}{\textbf{Method}} & 
\multicolumn{4}{c|}{\textbf{CRITEO-UPLIFT-Masked}} & 
\multicolumn{4}{c|}{\textbf{LAZADA-Masked}} & 
\multicolumn{4}{c}{\textbf{Production}}\\
\cmidrule(lr){2-5} \cmidrule(lr){6-9} \cmidrule(lr){10-13}
& LIFT$@30$ & AUUC & QINI & PUC & LIFT$@30$ & AUUC & QINI & PUC & LIFT$@30$ & AUUC & QINI & PUC \\
\midrule 
BaseNet+IPS & 0.3564 & 0.1332 & 0.1592 & 0.1376 & 0.0081 & 0.0030 & 0.0215 & \cellcolor{lightblue}{\textbf{0.0044}} & \cellcolor{lightblue}{\textbf{0.0230}} & 0.0119 & 0.0107 & 0.0113 \\
BaseNet+RD-IPS & 0.3568 & 0.1335 & 0.1597 & 0.1378 & 0.0082 & 0.0030 & 0.0215 & \cellcolor{lightblue}{\textbf{0.0044}} & 0.0227 & 0.0117 & 0.0106 & 0.0110 \\
BaseNet+RA-IPS & \cellcolor{lightblue}{\textbf{0.3580}} & \cellcolor{lightblue}{\textbf{0.1356}} & \cellcolor{lightblue}{\textbf{0.1621}} & \cellcolor{lightblue}{\textbf{0.1401}} & \cellcolor{lightblue}{\textbf{0.0083}} & \cellcolor{lightblue}{\textbf{0.0031}} & \cellcolor{lightblue}{\textbf{0.0217}} & \cellcolor{lightblue}{\textbf{0.0044}} & 0.0229 & \cellcolor{lightblue}{\textbf{0.0121}} & \cellcolor{lightblue}{\textbf{0.0109}} & \cellcolor{lightblue}{\textbf{0.0114}} \\
\midrule
CHAUN+IPS & 0.3612 & 0.1390 & 0.1658 & 0.1432 & 0.0083 & 0.0031 & 0.0221 & \cellcolor{lightblue}{\textbf{0.0047}} & 0.0257 & 0.0144 & 0.0130 & 0.0144\\
CHAUN+RD-IPS & 0.3619 & 0.1398 & 0.1668 & 0.1444 & 0.0083 & 0.0031 & 0.0219 & 0.0044 & 0.0256 & 0.0142 & 0.0129 & 0.0141\\
CHAUN+RA-IPS & \cellcolor{lightblue}{\textbf{0.3637}} & \cellcolor{lightblue}{\textbf{0.1430}} & \cellcolor{lightblue}{\textbf{0.1704}} & \cellcolor{lightblue}{\textbf{0.1474}} & \cellcolor{lightblue}{\textbf{0.0084}} & \cellcolor{lightblue}{\textbf{0.0033}} & \cellcolor{lightblue}{\textbf{0.0230}} & 0.0045 & \cellcolor{lightblue}{\textbf{0.0263}} & \cellcolor{lightblue}{\textbf{0.0151}} & \cellcolor{lightblue}{\textbf{0.0136}} & \cellcolor{lightblue}{\textbf{0.0156}}\\
\bottomrule[1.2pt]
\end{tabularx}
\end{adjustbox}
\end{table*}

\begin{figure}[t]
  \centering
  \includegraphics[
    width=0.49\textwidth,      
    trim={0pt 0pt 0pt 0pt}, 
    clip                      
  ]{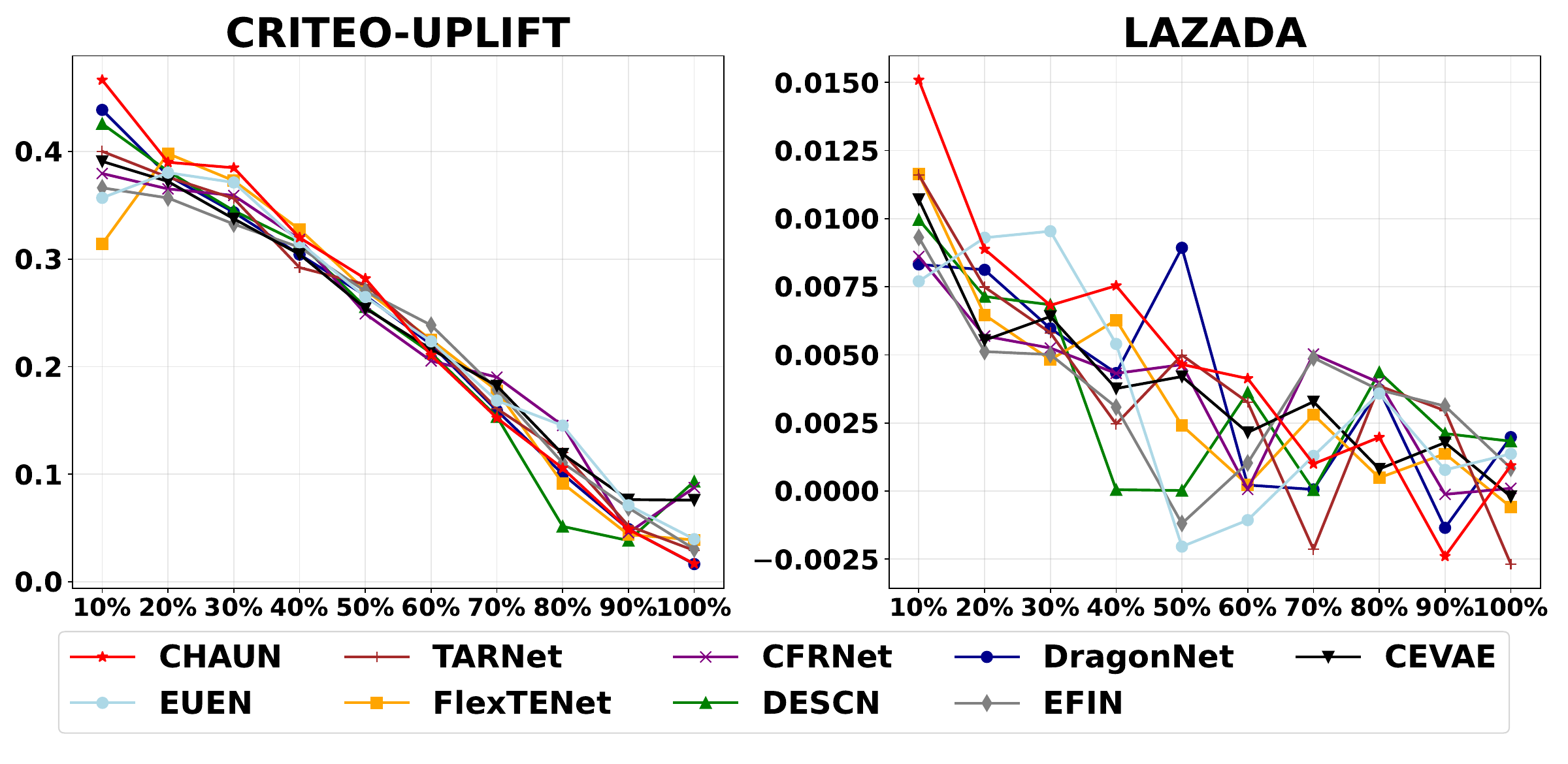}  
  \vspace{-5pt}
  \caption{The uplift curve on CRITEO-UPLIFT and LAZADA datasets. Samples are sorted in descending order based on their predicted uplift scores and divided into 10 decile groups. Then we calculate the actual uplift within each group. An ideal uplift curve exhibits strict monotonicity (steadily decreasing uplift across ranked deciles) and high discriminatory power, maximizing treatment benefits in top-ranked subgroups while minimizing harm in low-response groups.
}
  \label{curve}
\end{figure}

\subsection{In-Depth Analysis}

\subsubsection{Ablation Study}
\begin{table*}[h]  
  \centering
  \caption{Ablation study of CHAUN. BaseNet denotes a simplified variant of CHAUN with the attention module removed, while MMoE replaces the attention weighting mechanism with task-specific gating operations. All experiments were conducted with 5 random seeds, with metric averages computed for robustness evaluation. The best results of each benchmark are in \cellcolor{lightblue}{\textbf{bold}}.}
  \vspace{-9pt}
  \label{performance ablation}
  \begin{adjustbox}{width=0.98\textwidth, keepaspectratio}
  \small 
\begin{tabularx}{\textwidth}{
    >{\hsize=3.4\hsize\centering\arraybackslash}X
    *{3}{
    | >{\hsize=0.8\hsize\centering\arraybackslash}X  
      >{\hsize=0.8\hsize\centering\arraybackslash}X  
      >{\hsize=0.8\hsize\centering\arraybackslash}X
      >{\hsize=0.8\hsize\centering\arraybackslash}X  }}
\toprule[1.2pt]
\multirow{2}{*}{\textbf{Method}} & 
\multicolumn{4}{c|}{\textbf{CRITEO-UPLIFT}} & 
\multicolumn{4}{c|}{\textbf{LAZADA}} & 
\multicolumn{4}{c}{\textbf{Production}}\\
\cmidrule(lr){2-5} \cmidrule(lr){6-9} \cmidrule(lr){10-13}
& LIFT$@30$ & AUUC & QINI & PUC & LIFT$@30$ & AUUC & QINI & PUC & LIFT$@30$ & AUUC & QINI & PUC \\
\midrule 
BaseNet (removed attention)         & 0.3861 & 0.1597 & 0.1915 & 0.1641 & 0.0083 & 0.0041 & 0.0236 & 0.0065 & 0.0230 & 0.0107 & 0.0119 & 0.0113 \\
MMoE (replaced attention)         & 0.3891 & 0.1616 & 0.1940 & \cellcolor{lightblue}\textbf{0.1664} & 0.0084 & 0.0037 & 0.0262 & 0.0057 & 0.0249 & 0.0137 & 0.0124 & 0.0131 \\
CHAUN          & \cellcolor{lightblue}{\textbf{0.3904}} & \cellcolor{lightblue}{\textbf{0.1620}} & \cellcolor{lightblue}{\textbf{0.1955}} & \cellcolor{lightblue}{\textbf{0.1664}} & \cellcolor{lightblue}{\textbf{0.0087}} & \cellcolor{lightblue}{\textbf{0.0044}} & \cellcolor{lightblue}{\textbf{0.0314}} & \cellcolor{lightblue}{\textbf{0.0091}} & \cellcolor{lightblue}{\textbf{0.0261}} & \cellcolor{lightblue}{\textbf{0.0144}} & \cellcolor{lightblue}{\textbf{0.0130}} & \cellcolor{lightblue}{\textbf{0.0144}} \\
\bottomrule[1.2pt]
\end{tabularx}
\end{adjustbox}
\end{table*}

We conduct ablation studies to validate the efficacy of the novel cross-head attention module within CHAUN. A comparative baseline involves removing the attention module, which corresponds to the aforementioned BaseNet architecture. In addition, we consider another method analogous to the weighted summation in the attention module as a baseline, Multi-gate Mixture-of-Experts (MMoE)~\cite{MMOE}. MMoE utilizes multiple expert networks and trains task-specific gating functions to compute weighted combinations of their outputs. The key distinction of CHAUN's attention mechanism lies in its explicit exploitation of inter-group similarity between the treatment and control groups, where attention weights are derived through inner product similarity computations. To ensure a fair comparison, we configure MMoE with 2 experts for potential outcome prediction and calibrate the number of layers across models to maintain comparable parameter scales. The results of the ablation experiments are presented in \cref{performance ablation}. As evidenced by the results, MMoE consistently demonstrates superior performance compared to BaseNet across the majority of experimental scenarios. Notably, CHAUN achieves the best performance across all three datasets. On the CRITEO-UPLIFT dataset with low-dimensional features and pronounced treatment effects, CHAUN performs comparably to MMoE. However, on datasets featuring high-dimensional and complex feature spaces (LAZADA and Production), CHAUN demonstrates significantly stronger advantages, achieving QINI coefficient improvements of up to $19.8\%$.

\subsubsection{Visualization of Prediction}

\begin{figure}[t]
    \centering
    \begin{subfigure}[b]{0.23\textwidth}
        \includegraphics[
            trim=17pt 0 0 10pt, 
            clip,
            width=\textwidth
        ]{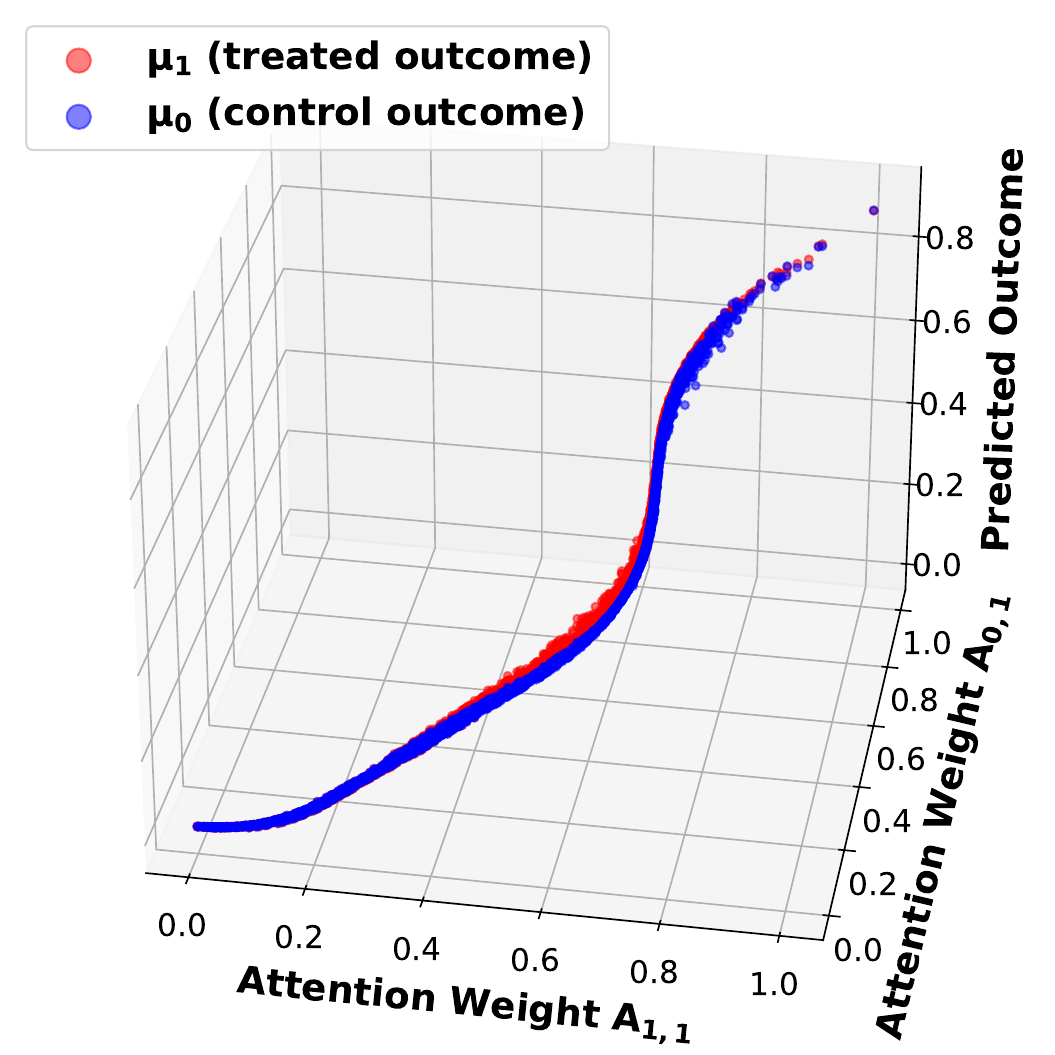}
        \caption{Predicted outcome.}
        \label{pred_out}
    \end{subfigure}
    \hfill 
    \begin{subfigure}[b]{0.23\textwidth}
        \includegraphics[width=\textwidth]{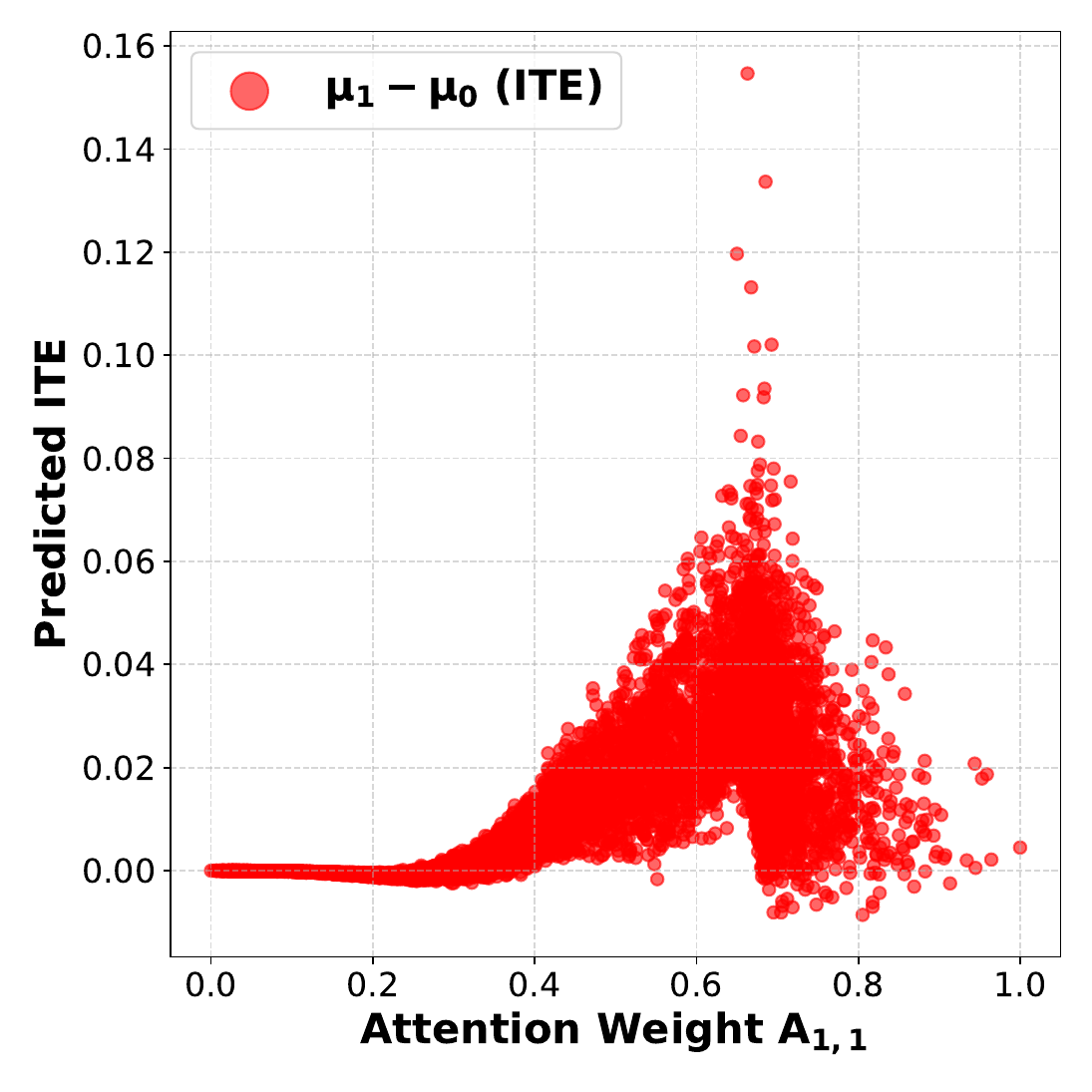}
        \caption{Predicted ITE.}
        \label{pred_ite}
    \end{subfigure}
    \caption{Scatter plots illustrating the relationship between predicted outcomes/predicted ITEs and attention weights. $\mathbf{A}_{1,1}$ denotes the attention weights where treatment group representations serve as both queries and keys, while $\mathbf{A}_{0,1}$ represent the weights where control group representations act as queries and treatment group representations as keys.}
    \label{pred}
\end{figure}
To clarify how the core attention module of CHAUN leverages inter-group similarity to enhance the discriminative power of predicted ITEs, we rigorously investigate the relationship between predicted outcomes, ITEs and computed attention weights on the Production dataset, with results visualized in \cref{pred}. As shown in \cref{pred_out}, the attention weights computed when using treatment group representations versus control group representations as queries (on x-axis and y-axis, respectively) exhibit close alignment. Furthermore, both predicted potential outcomes $\mu_0$ and $\mu_1$ increase monotonically with higher attention weights assigned to treatment group representations, i.e., greater focus on treatment group representations leads to larger predicted potential outcomes. 
Combined with \cref{pred_ite}, we observe that when attention weights heavily focus on control group representations, both $\mu_0$
and $\mu_1$ collapse near zero, leading to nearly all predicted ITEs being trivial. Conversely, when attention predominantly emphasizes treatment group representations, $\mu_0$ and $\mu_1$ exhibit large values with divergent ITEs, resembling predictions from independent model heads without inter-group interaction. Notably, the most significant ITE estimates emerge within an intermediate range of attention weights, where the model balances inter-group feature sharing and treatment-specific heterogeneity, thereby amplifying contrastive signals between counterfactual outcomes.
Based on these observations, we conclude that representations from different treatment groups are mapped to shared query spaces but distinct key and value spaces, preserving treatment-specific patterns. The attention weights dynamically calibrate inter-group interactions, enhancing ITE discriminability through amplified counterfactual contrasts.

\subsubsection{Hyperparameter Sensitivity}

\begin{figure}[t]
  \centering
  \includegraphics[
    width=0.49\textwidth,      
    trim={0pt 0pt 0pt 0pt}, 
    clip                      
  ]{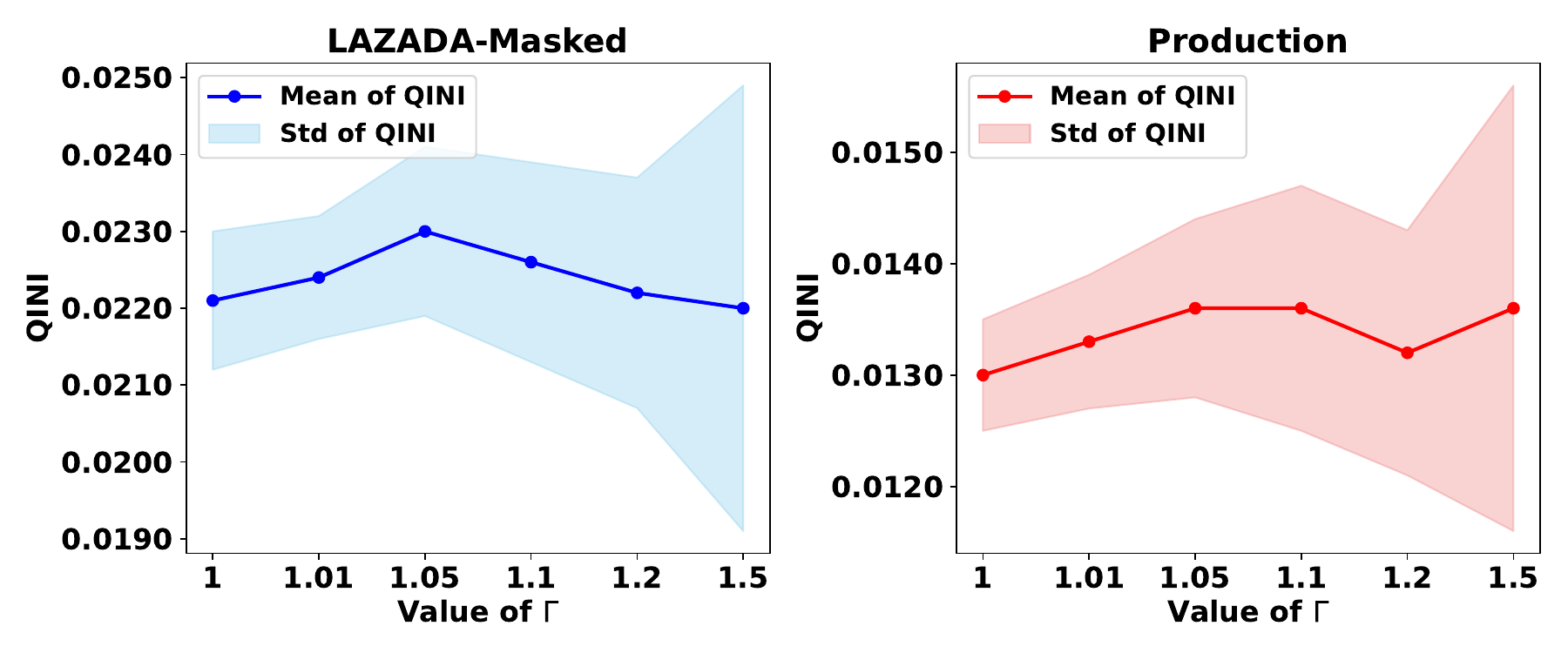}  
  \caption{Performance (QINI) of RA-IPS as the hyperparameter 
$\Gamma$ varies. The mean and standard deviation are computed from results obtained with five different random seeds. 
}
  \label{hyperfig}
\end{figure}

The RA-IPS method requires selecting an appropriate hyperparameter $\Gamma$ corresponding to the tolerable strength of unobserved confounders. We investigate how the performance of RA-IPS varies under different values of $\Gamma$, and present the results on LAZADA-Masked and Production datasets in \cref{hyperfig}. $\Gamma=1$ corresponds to the standard IPS. It can be observed that when $\Gamma$ varies within a small range, RA-IPS consistently achieves stable improvements over IPS. However, as the value of $\Gamma$ gradually increases, the performance begins to fluctuate noticeably, making it no longer guaranteed to outperform IPS, while also exhibiting a larger standard deviation. This suggests that when applying RA-IPS, if the strength of unobserved confounders is completely unknown, it is advisable to start with a relatively small $\Gamma$.

\section{Conclusions}
In this paper, we address the lack of flexible integration of inter-group correlations in uplift modeling by proposing CHAUN, a concise yet effective framework that dynamically computes inter-group attention scores to model such associations. Furthermore, for uplift scenarios with unobserved confounders, we establish that knowing the true propensity score for each unit serves as a sufficient condition to eliminate confounding bias. For practical settings where true propensity scores are unavailable, we introduce RA-IPS, which enhances robustness against unobserved confounders through adversarial learning over plausible propensity score spaces. Extensive comparative and ablation experiments demonstrate CHAUN's effectiveness as a general-purpose uplift model and RA-IPS's consistent improvements over conventional IPS methods in the presence of unobserved confounders.

\section*{Acknowledgements}
This work was supported by the NSFC Project (No.62576346), the MOE Project of Key Research
Institute of Humanities and Social Sciences (22JJD110001), the fundamental research funds for the
central universities, and the research funds of Renmin University of China (24XNKJ13), and Beijing
Advanced Innovation Center for Future Blockchain and Privacy Computing.



\appendix

\section{Theoretical Proof}
\label{proof}
\begin{theorem*}
[\cref{ITEwithpi}]
Under the assumption of the presence of unobserved confounders u with known true propensity scores $\pi(x,u)$,
the ITE can be identified through two approaches:
\begin{enumerate}[label=\arabic*),leftmargin=*,  
    labelwidth=1.5em,  
    align=left,         
    itemindent=0pt,     
    labelsep=0.5em]
\item 
\begin{equation*}
    ITE(x)=\mathbb{E}\left( \frac{ty}{\pi(x,u)}- \frac{(1-t)y}{1-\pi(x,u)}\bigg| x \right).
\end{equation*}

\item If the nominal propensity score $\tilde{\pi}(x)$ is known:
\begin{equation*}
    ITE(x)=\mathbb{E}\left(\frac{\tilde{\pi}(x)}{\pi(x,u)}y\bigg|x,t=1\right)-\mathbb{E}\left(\frac{1-\tilde{\pi}(x)}{1-\pi(x,u)}y\bigg|x,t=0\right).
\end{equation*}
\end{enumerate}
\end{theorem*}
\begin{proof}

\begin{enumerate}[label=\arabic*),leftmargin=*,  
    labelwidth=1.5em,  
    align=left,         
    itemindent=0pt,     
    labelsep=0.5em]
\item
By the Law of Total Expectation,
\begin{equation*}
\begin{split}
  \mathbb{E}(y(1)|x) &= \int \mathbb{E}(y(1)|x,u)p(u|x) \, du \\
&= \int t\mathbb{E}(y(1)|x,u)p(u,t|x)\frac{1}{\pi(x,u)} \, du \, dt \\
&= \mathbb{E}\left( \frac{ty(1)}{\pi(x,u)} \bigg| x \right)=\mathbb{E}\left( \frac{ty}{\pi(x,u)} \bigg| x \right),
  \end{split}
\end{equation*}
analogously, we obtain: 
\begin{equation*}
\mathbb{E}(y(0)|x)=\mathbb{E}\left( \frac{(1-t)y}{1-\pi(x,u)} \bigg| x \right),
\end{equation*}
thus, 
\begin{equation*}
ITE(x)=\mathbb{E}\left( \frac{ty}{\pi(x,u)}- \frac{(1-t)y}{1-\pi(x,u)}\bigg| x \right).
\end{equation*}

\item
$\begin{aligned}[t]
    \mathbb{E}(y(1)|x) &= \int \mathbb{E}(y(1)|x,u)p(u|x) du \\
    &=\int \mathbb{E}(y|x,u,t=1)\frac{\tilde{\pi}(x)}{\pi(x,u)}p(u|x,t=1)du \\
    &=\mathbb{E}\left(\frac{\tilde{\pi}(x)}{\pi(x,u)}y\bigg|x,t=1\right),
\end{aligned}$

\vspace{6pt}
analogously, we obtain: 
\begin{equation*}
\mathbb{E}(y(0)|x)=\mathbb{E}\left(\frac{1-\tilde{\pi}(x)}{1-\pi(x,u)}y\bigg|x,t=0\right),
\end{equation*}
thus,
\begin{equation*}
ITE(x)=\mathbb{E}\left(\frac{\tilde{\pi}(x)}{\pi(x,u)}y\bigg|x,t=1\right)-\mathbb{E}\left(\frac{1-\tilde{\pi}(x)}{1-\pi(x,u)}y\bigg|x,t=0\right). \qedhere
\end{equation*}
\end{enumerate}
\end{proof}

\begin{theorem*}[\cref{unbias}]
    Under the assumption of the presence of unobserved confounders u, let $\mathcal{L}_{true-IPS}$ denote the IPS estimator using the true propensity score $\pi(x_i,u_i)$, formally expressed as 
\[
\mathcal{L}_{true-IPS}=\frac1N\sum_{i=1}^N \bigg(\frac{(1-t_i)l(\mu_0(x_i),y_i(0))}{1-\pi(x_i,u_i)}+  \frac{t_il(\mu_1(x_i),y_i(1))}{\pi(x_i,u_i)}\bigg).
\]
Then it is an unbiased estimator of $\mathcal{L}_{ideal}$.
\end{theorem*}
\begin{proof}
For each sample,
\begin{align*}
    &\mathbb{E}_{t_i}\bigg(\frac{(1-t_i)l(\mu_0(x_i),y_i(0))}{1-\pi(x_i,u_i)}+  \frac{t_il(\mu_1(x_i),y_i(1))}{\pi(x_i,u_i)}\bigg)\\&=(1-\pi(x_i,u_i))\frac{l(\mu_0(x_i),y_i(0))}{1-\pi(x_i,u_i)}+\pi(x_i,u_i)\frac{l(\mu_1(x_i),y_i(1))}{\pi(x_i,u_i)}\\
    &=l(\mu_0(x_i),y_i(0))+l(\mu_1(x_i),y_i(1))
\end{align*}
Since treatment assignments are independent across all samples, we have 
    \begin{align*}
        &\mathbb{E}_t(\mathcal{L}_{true-IPS})\\&=\frac1N\sum_{i=1}^N \mathbb{E}_{t_i}\bigg(\frac{(1-t_i)l(\mu_0(x_i),y_i(0))}{1-\pi(x_i,u_i)}+  \frac{t_il(\mu_1(x_i),y_i(1))}{\pi(x_i,u_i)}\bigg)\\
        &=\frac1N\sum_{i=1}^N(l(\mu_0(x_i),y_i(0))+l(\mu_1(x_i),y_i(1)))=\mathcal{L}_{ideal}. \qedhere
    \end{align*}
    
\end{proof}


\begin{proposition*} [\cref{true prop limit}]
    For a fixed $x$, suppose that $a(x)<\pi(x,u)<1-a(x), \forall u$, then the following inequalities each hold with probability at most $\eta$,
\begin{align*}
    |\frac{1}{\pi(x,u)}-\frac{1}{\tilde{\pi}(x)}|&\ge \frac{1-2a(x)}{2a(x)\tilde{\pi}(x)\sqrt\eta},\\
|\frac{1}{1-\pi(x,u)}-\frac{1}{1-\tilde{\pi}(x)}|&\ge \frac{1-2a(x)}{2(1-a(x))\tilde{\pi}(x)\sqrt\eta}.
\end{align*}
\end{proposition*}
\begin{proof}
Since $a(x)<\pi(x,u)<1-a(x)$, we have
\begin{align}
    &(\pi(x,u)-a(x))(\pi(x,u)-(1-a(x)))\le0 \notag\\
    \quad\Leftrightarrow\quad &\pi(x,u)^2-\pi(x,u)+a(x)(1-a(x))\le0 \label{e2}
\end{align}
Taking the conditional expectation of \cref{e2} with respect to $x$, we obtain:
\begin{align}
    &\mathbb{E}_{u|x}(\pi(x,u)^2)\le \tilde{\pi}(x)-a(x)(1-a(x)) \notag\\
    \quad\Leftrightarrow\quad &\text{Var}_{u|x}(\pi(x,u))\le - \tilde{\pi}(x)^2+\tilde{\pi}(x)-a(x)(1-a(x)) \label{varu}
\end{align}
The right-hand side of \cref{varu} attains its maximum $(a(x)-\frac12)^2$ at $\tilde{\pi}(x)=\frac12 $. Notice that
\begin{align*}
    &|\frac{1}{\pi(x,u)}-\frac{1}{\tilde{\pi}(x)}|\ge\epsilon \\
    \quad\Leftrightarrow\quad &|\pi(x,u)-\tilde{\pi}(x)|\ge \epsilon \pi(x,u) \tilde{\pi}(x) \\
    \Rightarrow\quad &|\pi(x,u)-\tilde{\pi}(x)|\ge \epsilon a(x) \tilde{\pi}(x)
\end{align*}
By Chebyshev's Inequality, we have
\begin{align*}
    \Pr(|\frac{1}{\pi(x,u)}-\frac{1}{\tilde{\pi}(x)}|\ge\epsilon)\le \Pr(|\pi(x,u)-\tilde{\pi}(x)|\ge \epsilon a(x) \tilde{\pi}(x)) \\
     \le \frac{\text{Var}_{u|x}(\pi(x,u))}{(\epsilon a(x) \tilde{\pi}(x)))^2}
    \le \bigg(\frac{\frac12-a(x)}{\epsilon a(x) \tilde{\pi}(x))}\bigg)^2.
\end{align*}
Letting $\eta = (\frac{\frac12-a(x)}{\epsilon a(x) \tilde{\pi}(x))})^2$, then we have with probability at most $\eta$, 
\begin{align*}
    |\frac{1}{\pi(x,u)}-\frac{1}{\tilde{\pi}(x)}|&\ge \frac{1-2a(x)}{2a(x)\tilde{\pi}(x)\sqrt\eta}.
\end{align*}
The other inequality can be analogously proved following the same methodology.
\end{proof}



\begin{theorem*}
[\cref{bound}]
Suppose that the true weight configurations lie within the uncertainty set in \cref{set}, and $l_i <C_1,w_i<C_2$. For any $f_\phi \in \mathcal{F}$, we have that with probabilty at least $1-\eta$,
\begin{equation*}
    |\mathcal{L}_{RA-IPS}(\phi) - \mathcal{L}_{ideal}(\phi)| 
    \leq 2 C_2 \mathcal{R}(\mathcal{F}) + C_1 C_2(1+ \sqrt{\frac{\log(2/\eta)}{2N}})
\end{equation*}
\end{theorem*}
\begin{proof}
$
    |\mathcal{L}_{RA-IPS}(\phi)-\mathcal{L}_{ideal}(\phi)|\le|\mathcal{L}_{RA-IPS}(\phi)-\mathcal{L}_{true-IPS}(\phi)|\\+|\mathcal{L}_{true-IPS}(\phi)-\mathcal{L}_{ideal}(\phi)|.
$
Next, we derive upper bounds for each of these two terms individually. Since the true weighting combination lies within our predefined uncertainty set, for the first term we have
    $|\mathcal{L}_{RA-IPS}(\phi)-\mathcal{L}_{true-IPS}(\phi)|=\mathcal{L}_{RA-IPS}(\phi)-\mathcal{L}_{true-IPS}(\phi)
    =\frac1N\sum_{i=1}^{N} (\hat{w}_i-w_i)l_i
    \le C_1C_2$
We now bound the second term $|\mathcal{L}_{\text{true-IPS}}(\phi) - \mathcal{L}_{\text{ideal}}(\phi)|$. By our assumption, the per-sample loss and the true weights are bounded: $l_i \leq C_1$ and $w_i \leq C_2$. Therefore, each weighted loss term $w_i l_i$ is bounded by $M = C_1 C_2$.

Consider the random variable $\Phi(\mathcal{D}) = \sup_{f_\psi \in \mathcal{F}} |\mathcal{L}_{ideal}(f_\psi) - \mathcal{L}_{true-IPS}(f_\psi)|$. Replacing any single data point $(x_i, t_i, y_i)$ in the sample set $\mathcal{D}$ can change $\mathcal{L}_{true-IPS}(f_\psi)$ by at most $M/N$, and this bounded difference property holds for the supremum $\Phi(\mathcal{D})$. By McDiarmid’s Inequality~\cite{mcdiarmid}, with probability at least $1 - \eta$,
\[
\Phi(\mathcal{D}) \leq \mathbb{E}_\mathcal{D}[\Phi(\mathcal{D})] + M \sqrt{\frac{\log(2/\eta)}{2N}}.
\]

By Symmetrization Lemma~\cite{radamacher}, the expected uniform deviation is bounded by twice the Rademacher complexity of the weighted loss class. Since the nominal weights are bounded by $C_2$, the contraction property of Rademacher complexity implies that this quantity is at most $2 C_2 \mathcal{R}(\mathcal{F})$. Hence,
\[
\mathbb{E}_{\mathcal{D}}[\Phi(\mathcal{D})] \leq 2 C_2 \mathcal{R}(\mathcal{F}).
\]

Combining these results, we conclude that with probability at least $1 - \eta$,
\begin{equation*}
    |\mathcal{L}_{RA-IPS}(\phi) - \mathcal{L}_{ideal}(\phi)| 
    \leq 2 C_2 \mathcal{R}(\mathcal{F}) + C_1 C_2(1+ \sqrt{\frac{\log(2/\eta)}{2N}}).
\end{equation*}

\end{proof}

\bibliographystyle{ACM-Reference-Format}
\bibliography{sample-base}

\end{document}